\documentclass[11pt]{amsart}

\usepackage{amsmath,amssymb,amsthm}
\usepackage{graphicx}
\usepackage{hyperref}

\usepackage{algorithm}
\usepackage{algorithmic}
\usepackage{lineno} 
\usepackage{geometry}
\usepackage{tikz}
\usepackage{pgfplots}

\usepackage{tikz-cd}

\usepackage{microtype}

\theoremstyle{plain} 
\newtheorem{theorem}{Theorem}
\newtheorem{proposition}{Proposition}

\theoremstyle{definition}

\theoremstyle{remark} 
\newtheorem{remark}[theorem]{Remark} 

\def\eq#1{(\ref{#1})}

\def\sech{ {\rm sech}}

\def\beq{\begin{equation}}
\def\eeq{\end{equation}}

\title{A  Dynamical Framework for Cognitive Processes Based on Transformations and Semantic Equivalence}
\author{Carlo Cattani$^{1,*}$ \and Dioneia Motta Monte-Serrat$^{2}$ }
\date{}

\begin{document}

\maketitle

\begin{center}	
	$^1$ Engineering School, DEIM,  University of Tuscia, (VT), 01100, Italy
	Email: cattani@unitus.it 
	\\
	$^2$ Department of Physics, University of Sao Paulo, USP; Department of Law, University of Ribeirao Preto, Unaerp, Brazil. Email: dserrat@unaerp.br
	 \\
	$^{*}$ Corresponding author    
\end{center}

	\begin{abstract}
		
		This paper proposes a structural and dynamical framework for modeling cognitive processes within a cybernetic perspective. Cognitive states are represented as elements of a state space evolving through an iterative update rule of the form
		\[
		X_{t+1} = \pi\big(F(f(X_t))\big),
		\]
		where $f$ describes internal transformations, $F$ represents interpretative mappings, and $\pi$ enforces semantic equivalence.
			The model is interpreted as a feedback system integrating transformation, observation, and stabilization. A categorical formulation is introduced to capture compositional structure, while the associated dynamics are analyzed through fixed-point arguments and contraction conditions ensuring stability.
			To demonstrate the operational character of the framework, a computational illustration is provided, together with a qualitative analysis of the induced dynamics. A concrete linguistic application shows how context-dependent interpretation can be modeled as a trajectory toward a stable semantic class.
			The proposed approach connects dynamical systems, category theory, and cognitive modeling, and provides a unified representation of cognition as a feedback-driven process evolving toward invariant interpretations.
\\\\
\textbf{Keywords:} 
	Cybernetics, cognitive systems, dynamical systems, feedback, category theory, semantic stabilization.
\\\\
\textbf{AMS classification:} \subjclass[2020]{
	37N99,   
	93C55,   
	91E10,   
	18B99,   
	68T27    
}
\end{abstract}


\section{Introduction}  

The mathematical modeling of cognition and language presents challenges due to the complexity offered by the multiple concepts of language and the partial  knowledge about the functioning of cognition. The cognitive linguistic process follows a single dynamic, which articulates logical reasoning with context-sensitive data. 

Understanding cognition requires integrating logical reasoning with the contextual processes through which meaning is 
continuously updated. Classical models of reasoning emphasize propositional structures of the form
\begin{equation}
P \Rightarrow Q
\end{equation}
which represent deductive inference.
The logical aspect is easily represented under the hypothesis ``if P, then Q'', which presupposes an interpretative path detached from the context. This is a general pattern of reasoning, without reference to a specific meaning or context. Logical-mathematical reasoning defines a proposition as a sentence that is either true or false (but not both) \cite{ecc13}. 
However, according to Whorf \cite{who12}, language symbolizes processes of reference and logic, making sentence structure patterns more important in constructing information than the words themselves. Derrida \cite{der68}, in turn, states that grammar (syntax) determines what must be true to represent and confer meaning determining another thought beforehand \cite{der68}, replacing semantic aspects of information through standardization and anticipation.

 While powerful, such logical structures capture only the static dimension of reasoning, and do not allow for the continuously changing representation of the environment's evolution.
In fact, human cognition instead operates dynamically, constantly updating interpretations based on contextual input \cite{moca21,moca21a}. 
To model this process mathematically we propose representing cognition through transformations between cognitive states. Cognition may be broadly defined as the set of processes through which an organism acquires, represents, transforms, and uses information about the world. These processes include perception, learning, memory, reasoning, language, and decision making \cite{Anderson2015,Neisser1967}. In contemporary cognitive science, cognition is often understood as a dynamic system in which internal representations are continuously updated in response to sensory input and contextual information \cite{Friston2010}. From a mathematical perspective, cognitive activity can therefore be modeled as a sequence of transformations acting on representational states, where information is processed, structured, and stabilized through interaction with the environment.

Many authors have already  used  categories to describe universal constructs \cite{MacLane1971} that lead to information. For instance, Al Osaimi  addresses variations in identity and facial expression, under equivalence relations between the descriptors; invariant adjustments are related to rotation \cite{alo20}. 

Invariant learning reflects aspects of human cognition to describe key points in a set of descriptors. The dynamic aspects of semantics can be represented through equivalences and invariants, optimizing the performance of intelligent systems.

Function theory, by grouping elements into categories, represents processes rather than objects \cite{lawvere1963,lawvere2009}, resulting in a unified recurrent form of adjoint functors that formalize semantics as a functor capable of providing new characterizations of abstract categories. According to Clark \cite[p.47]{clark1998}, mind, body, and world act as equal partners in determining cognition or behavior, which aligns with neuroscience \cite{per07}.


In recent years, category theory has increasingly been used to study semantic structure, cognition, and compositional systems. 
Several recent contributions (2023--2025) demonstrate the growing importance of categorical methods for modeling meaning, representation, and reasoning. Quigley develops a categorical framework for extensional and intensional models in formal semantics, showing that trivial intensional models are equivalent to extensional ones within a typed categorical setting \cite{Quigley2024}. 
This work provides a precise bridge between classical formal semantics and categorical structure. In his paper, Phillips proposes a category-theoretic interpretation of the Language of Thought hypothesis and shows how presheaves and universal constructions can model relationships between symbolic and non-symbolic representations \cite{Phillips2024}. 
The paper provides an explicit cognitive interpretation of categorical structure.
Prentner argues that category theory offers a structural framework for consciousness science that goes beyond purely correlational models \cite{Prentner2024}. 
The work highlights how categorical relations can capture explanatory structures in cognitive systems. Lambert develops a topos-theoretic semantics for intuitionistic modal logic and applies it to branching spacetime models \cite{Lambert2024}. 
This work demonstrates how contextual and modal notions of truth can be modeled within presheaf topoi.
Lewis investigates compositional vector semantics from a categorical perspective and proposes methods for grounding compositional representations in neural architectures \cite{Lewis2024}. 
The paper connects categorical semantics with cognitively plausible machine learning models.
Wojtowicz proposes logic-based artificial intelligence algorithms operating within categorical semantic structures \cite{Wojtowicz2025}. 
The work demonstrates how categorical semantics can support reasoning procedures beyond classical symbolic frameworks.
Bakirtzis et al.\ develop categorical semantics for compositional reinforcement learning using categorical representations of Markov decision processes \cite{Bakirtzis2025}. 
Their work illustrates how categorical structure supports modular reasoning in adaptive systems. Mahadevan explores the potential of topos theory as a foundation for generative AI and large language models \cite{Mahadevan2025}. 
In this  paper it is suggested that richer categorical structures may improve contextual reasoning and compositional representation in AI.

In this paper we define a  theoretical framework that  integrates logical reasoning, functional transformations, categorical structure, and contextual dynamics into a unified mathematical representation of cognition. Logical inference provides the initial symbolic structure of reasoning, while cognitive transformations update representations through contextual interaction. Equivalence relations identify invariant semantic structures across representational variations, and categorical mappings organize these transformations into compositional systems. The contextual logic of a topos ensures coherence across different interpretative contexts, while temporal dynamics model the evolution of cognitive states over time.

The dynamical model proposed in this article describes cognition as an iterative process governed by the update rule
\beq
\label{xi1}
\framebox{$
X_{t+1} = \pi \big( F(f(X_t)) \big),
$}
\eeq
where $\pi : X \to X/{\sim}$ denotes the canonical projection onto equivalence classes (alternatively \eq{xi1} could be also written as $X_{t+1} = F(f(X_t)) / \sim$). 

Here, contextual transformations, semantic mappings, and equivalence relations interact to stabilize meaning over time.

\begin{remark}[Cybernetic interpretation]
	The update rule can be interpreted as a feedback system in the sense of cybernetics.
	
	The transformation $f$ represents the internal evolution of cognitive states, while $F$ acts as an interpretative or measurement operator. The projection $\pi$ enforces a form of semantic normalization by identifying equivalent representations.
	
	Thus, the overall system can be viewed as a closed feedback loop:
	\[
	X_t \;\longrightarrow\; f(X_t) \;\longrightarrow\; F(f(X_t)) \;\longrightarrow\; \pi(F(f(X_t))) \;\longrightarrow\; X_{t+1}.
	\]
	
	This structure reflects a fundamental cybernetic principle: the interaction between transformation, observation, and stabilization through feedback.
\end{remark}

Many artificial intelligence systems rely primarily on statistical pattern recognition and static symbolic representations. 
The structural framework proposed here suggests that robust cognitive systems must integrate logical inference, 
semantic invariance, and contextual updating within a unified mathematical structure.

The remainder of this paper is organized as follows: section 2 deals with some preliminary remarks, cognitive dynamics and temporal inlfuence is studied in section 3, sect. 4 investigates topos and sheaves for a topos theorical model temporal. The temporal dynamics is studied in section 5. In sect. 6 the axiomatic framework of the structural theory of cognition is given, while in sect. 7 several theorems on the property of the moel and stabilization   are shown. Sec. 8 studies the adjunction of this model by showing an universal property and at last in sect. 9 is explicitly given an extended application by showing formalization of this model.


\section{Categorial Representation of Cognition}

Semantics is the branch of linguistics and logic concerned with the study of meaning in language and symbolic systems. It investigates how expressions, sentences, and larger linguistic structures are interpreted and how they relate to objects, situations, or concepts in the world \cite{ ChierchiaMcConnellGinet2000,Lyons1977}. In formal approaches, semantics is typically modeled by assigning interpretations to syntactic expressions through mathematical structures that determine their truth conditions or conceptual content \cite{DowtyWallPeters1981}. From a cognitive perspective, semantics can therefore be understood as the structured mapping between representations and meanings that allows linguistic expressions to convey information and support inference.
	
Category theory provides a natural framework for representing transformations and invariants.
	In category theory, a \emph{category} consists of a collection of objects together with morphisms (or arrows) between them, representing relationships or transformations between objects. Formally, a category $\mathcal{C}$ is specified by a class of objects, a class of morphisms between objects, an associative composition law for morphisms, and identity morphisms for each object \cite{lawvere2009,MacLane1971}. Categories provide an abstract framework for studying structures and the mappings between them, allowing mathematical systems to be analyzed in terms of compositional transformations rather than the internal properties of individual objects.
	
A category $\mathcal{C}$ consists of: a class of objects,
  a class of morphisms between objects,
  an associative composition law for morphisms,
	  identity morphisms for each object.

	Thus cognition may be modeled as a category in which mental representations are objects and cognitive operations are morphisms.

Let \(A\) and \(B\) be sets of cognitive states. A transformation between them is represented by a function
\[
f : A \to B.
\]
 Thus cognition is modeled not as a collection of static propositions but as a system of transformations.


	In category theory, a \emph{morphism} (or arrow) represents a structured relationship or transformation between two objects of a category. 
	Given a category $\mathcal{C}$, a morphism
	
	\[
	f : X \rightarrow Y
	\]
	
	denotes a mapping from the object $X$ to the object $Y$. Morphisms can be composed whenever the codomain of one morphism coincides with the domain of another, and this composition is associative. Each object also possesses an identity morphism
	
\beq
\label{fxy}
	\mathrm{id}_X : X \rightarrow X
\eeq
		that acts as a neutral element for composition \cite{Awodey2010,MacLane1971}. Morphisms therefore capture the structural transformations that connect objects within a category.

\begin{remark}
	
Cognition often treats distinct representations as equivalent in meaning. 
Mathematically this corresponds to an equivalence relation
$
x \sim y
$
satisfying reflexivity, symmetry, and transitivity. Equivalence relations group representations into semantic classes that remain invariant across contextual variations.
\end{remark}

\subsection{Functorial Semantic Mapping}

In category theory, a \emph{functor} is a structure-preserving mapping between categories. 
Given two categories $\mathcal{C}$ and $\mathcal{D}$, a functor

\[
F : \mathcal{C} \rightarrow \mathcal{D}
\]
assigns to each object $X$ of $\mathcal{C}$ an object $F(X)$ of $\mathcal{D}$, and to each morphism 
$f : X \rightarrow Y$ in $\mathcal{C}$ a morphism

\[
F(f) : F(X) \rightarrow F(Y)
\]

in $\mathcal{D}$, such that identity morphisms and composition are preserved. 
In particular,

\[
F(\mathrm{id}_X) = \mathrm{id}_{F(X)}, 
\qquad
F(g \circ f) = F(g) \circ F(f)
\]
for all composable morphisms $f$ and $g$ in $\mathcal{C}$ \cite{Awodey2010,MacLane1971}. 
Functors therefore provide the basic mechanism for translating structures and relationships from one category into another.

In cognitive modeling:

\begin{itemize}
\item $\mathcal{C}$ represents cognitive representations
\item $\mathcal{D}$ represents semantic structures
\end{itemize}

Functors preserve structural relationships between representations.

	A \emph{contravariant functor} is a mapping between categories that reverses the direction of morphisms. 
	Given two categories $\mathcal{C}$ and $\mathcal{D}$, a contravariant functor
	
	\[
	F : \mathcal{C}^{op} \rightarrow \mathcal{D}
	\]
	
	assigns to each object $X$ of $\mathcal{C}$ an object $F(X)$ of $\mathcal{D}$, and to each morphism 
$
	f : X \rightarrow Y
$
	in $\mathcal{C}$ a morphism
	\[
	F(f) : F(Y) \rightarrow F(X)
	\]
	in $\mathcal{D}$, reversing the direction of the original morphism. 
	This  functor preserves identities but  reverses the order of composition, so that
	
	\[
	F(\mathrm{id}_X) = \mathrm{id}_{F(X)}, 
	\qquad
	F(g \circ f) = F(f) \circ F(g)
	\]
		for all composable morphisms $f$ and $g$ in $\mathcal{C}$ \cite{Awodey2010,MacLane1971}. 
	Contravariant functors play an important role in many areas of mathematics, particularly in the theory of presheaves, where structures defined on a category vary contravariantly with respect to morphisms.


\subsection{Cybernetic Interpretation and Computational Illustration}

\subsubsection{System interpretation}

The proposed model can be interpreted as a discrete-time cybernetic system in which cognitive states evolve under feedback.

In this perspective:
\begin{itemize}
	\item $f$ represents the internal processing dynamics,
	\item $F$ corresponds to an interpretative or evaluative mechanism,
	\item $\pi$ enforces equivalence by collapsing states into invariant semantic classes.
\end{itemize}

The system combines transformation and regulation, where the equivalence relation acts as a stabilizing mechanism.

\subsubsection{Computational example}

To illustrate the model, consider the case $X = \mathbb{R}$ and define:
\[
f(x) = x + \alpha \sin(x), \qquad F(x) = \tanh(x),
\]
with $\alpha > 0$.

Define the equivalence relation:
\[
x \sim y \quad \Longleftrightarrow \quad |x - y| < \varepsilon,
\]
for some $\varepsilon > 0$.

The iterative system becomes:
\[
X_{t+1} = \pi\big(\tanh(X_t + \alpha \sin(X_t))\big).
\]

This system exhibits the following properties:
\begin{itemize}
	\item bounded dynamics due to the saturation of $\tanh$,
	\item nonlinear oscillatory behavior induced by $\sin(x)$,
	\item stabilization through equivalence classes.
\end{itemize}

Depending on $\alpha$ and $\varepsilon$, the system may converge to fixed points or exhibit bounded fluctuations.

\subsubsection{Algorithmic implementation}

The model can be implemented through the following iterative procedure:

\begin{enumerate}
	\item Initialize $X_0 \in \mathbb{R}$,
	\item Compute $Y_t = f(X_t)$,
	\item Compute $Z_t = F(Y_t)$,
	\item Set $X_{t+1} = \pi(Z_t)$,
	\item Repeat until convergence.
\end{enumerate}

This demonstrates that the proposed framework is not purely abstract but admits concrete computational realization.


\section{Categorical Diagram of Cognitive Dynamics}


Let $X$ be a set and let $\sim$ be an equivalence relation on $X$, that is, a binary relation that is reflexive, symmetric, and transitive. 
The \emph{quotient set} of $X$ by the relation $\sim$, denoted

\[
X / \sim,
\]

is the set of equivalence classes of elements of $X$ under $\sim$ \cite{Halmos1960, MacLane1971}. 
Each element $x \in X$ determines an equivalence class

\[
[x] = \{\, y \in X \mid y \sim x \,\},
\]

and the quotient set is defined as

\[
X/\sim = \{\, [x] \mid x \in X \,\}.
\]

Intuitively, the quotient construction identifies elements that are considered equivalent under the relation $\sim$, thereby forming a new set whose elements represent classes of indistinguishable objects.	

 We show that the structural dynamics of cognition can be summarized by the update rule \eq{xi1}, and 
 the transformation process can be represented diagrammatically as

\[
\begin{tikzcd}
X_t \arrow[r,"f"] & f(X_t) \arrow[r,"F"] & F(f(X_t)) \arrow[r,"/\sim"] & X_{t+1}
\end{tikzcd}
\]

This diagram shows cognition as a sequence of compositional transformations.


\subsection{Temporal Dynamics and Coalgebra}

In order to show the validity of \eq{xi1} we need to define the influence of temporal dynamics on the semantic structure. 

	In category theory, a \emph{coalgebra} provides a mathematical framework for describing state-based dynamical systems. 
	Given an endofunctor 
	
	\[
	F : \mathcal{C} \rightarrow \mathcal{C},
	\]
	an $F$-coalgebra consists of an object $X$ of the category $\mathcal{C}$ together with a morphism
	
	\[
	\gamma : X \rightarrow F(X).
	\]
	
	The object $X$ represents the set of states of the system, while the structure map $\gamma$ describes how each state evolves or is observed through the functor $F$ \cite{Rutten2000}. Coalgebras are widely used to model evolving structures such as transition systems, streams, and dynamical processes, providing a categorical dual to algebraic constructions \cite{Rutten2000}. From this perspective, coalgebra offers a natural mathematical language for describing systems whose behavior is determined by transitions between states.

Cognitive systems evolve over time through state transitions. 
A dynamic cognitive system can be modeled coalgebraically as

$$
X \to F(X)
$$

where $X$ represents cognitive states and $F$ encodes the update rules governing cognitive evolution.


\section{Topos and Presheaf }

The category $\mathbf{Set}$ is the category whose objects are sets and whose morphisms are functions between sets. 
More precisely, for any two sets $X$ and $Y$, a morphism  
$
f : X \rightarrow Y
$
in $\mathbf{Set}$ is a function assigning to each element of $X$ a unique element of $Y$. 
Composition of morphisms corresponds to the usual composition of functions, and for each set $X$ there exists an identity morphism

\[
\mathrm{id}_X : X \rightarrow X
\]
that maps every element to itself \cite{Awodey2010,MacLane1971}. 
The category $\mathbf{Set}$ plays a fundamental role in category theory, serving as a canonical example of a category and as the target category for many functors, including presheaves of the form
$$
P : \mathcal{C}^{op} \rightarrow \mathbf{Set}.
$$
So that, a \emph{presheaf} on a category $\mathcal{C}$ is a contravariant functor
$
P : \mathcal{C}^{op} \rightarrow \mathbf{Set}.
$
This means that to each object $U$ of $\mathcal{C}$ the presheaf assigns a set $P(U)$, and to each morphism
\[
f : V \rightarrow U
\]
in $\mathcal{C}$ it assigns a restriction map
\[
P(f) : P(U) \rightarrow P(V),
\]
such that identity morphisms and compositions are preserved \cite{Johnstone2002,MacLaneMoerdijk1994}. Intuitively, a presheaf assigns data to each object of the category together with consistent restriction maps that relate the data across morphisms. Presheaves play a central role in modern geometry and topos theory, where they are used to represent structures that vary across contexts.

Sheafification is the process that associates to a presheaf the closest sheaf satisfying the gluing and locality conditions imposed by a Grothendieck topology. More precisely, given a site $(\mathcal{C},J)$ and a presheaf 
$
P : \mathcal{C}^{op} \rightarrow \mathbf{Set},
$
the \emph{sheafification} of $P$ is a sheaf $a(P)$ together with a natural morphism
\[
\eta : P \rightarrow a(P)
\]
such that $a(P)$ is universal among sheaves receiving a morphism from $P$ \cite{Johnstone2002,MacLaneMoerdijk1994}. In categorical terms, sheafification defines a functor
\[
a : \widehat{\mathcal{C}} \rightarrow \mathbf{Sh}(\mathcal{C},J)
\]
that is left adjoint to the inclusion of the category of sheaves into the category of presheaves. Intuitively, sheafification enforces the coherence conditions required for local data to determine consistent global structures.

\begin{remark}
		The  \emph{Grothendieck topology} provides a categorical generalization of the notion of open coverings in topology. 
		Given a category $\mathcal{C}$, a Grothendieck topology assigns to each object $U$ of $\mathcal{C}$ a collection of families of morphisms	
		\beq
		\label{fiui}
		\{\, f_i : U_i \rightarrow U \,\}
		\eeq
		called \emph{covering families}, which satisfy stability and compatibility conditions analogous to those of open covers in classical topology \cite{Johnstone2002,MacLaneMoerdijk1994}. 
		These conditions ensure that coverings are stable under pullback and that coverings can be refined by other coverings. 
		Grothendieck topologies allow the definition of sheaves on arbitrary categories and play a fundamental role in the construction of topoi \cite{MacLaneMoerdijk1994}.
		Let $\mathcal{C}$ be a category of contexts (or cognitive situations).
		Equipping $\mathcal{C}$ with a Grothendieck topology $J$ yields a site $(\mathcal{C},J)$.
		The covering family \eq{fiui}
		is interpreted as a collection of locally accessible perspectives that jointly determine the global situation $U$.
		This formalizes the principle that cognition often proceeds by assembling coherent global interpretations from locally partial evidence.
		In comparison,  the  \emph{Boolean algebra} is a complemented distributive lattice that provides an algebraic structure for classical propositional logic. Formally, a Boolean algebra is a set $B$ equipped with two binary operations, meet $(\wedge)$ and join $(\vee)$, a unary complement operation $(\neg)$, and two distinguished elements $0$ and $1$, representing the least and greatest elements respectively. These operations satisfy the laws of commutativity, associativity, distributivity, identity, and complementation \cite{DaveyPriestley2002,Halmos1963}. Boolean algebras serve as the algebraic semantics of classical logic, where propositions take values in $\{0,1\}$ corresponding to false and true, and logical operations correspond to algebraic operations within the lattice.
\end{remark}

In category theory, a \emph{topos} is a category that behaves in many respects like the category of sets while supporting an internal logical structure. More precisely, a topos is a category possessing finite limits, exponentials, and a subobject classifier, which together allow the interpretation of logical propositions within the category itself \cite{MacLaneMoerdijk1994, Johnstone2002}. Because of this internal logic, topoi provide a natural mathematical framework for modeling structures in which truth and validity may depend on contextual conditions \cite{bell2008}.

An important example is the \emph{presheaf topos}. Given a small category $\mathcal{C}$, the category

\[
\mathbf{Set}^{\mathcal{C}^{op}}
\]
of contravariant functors from $\mathcal{C}$ to the category of sets is called the presheaf topos over $\mathcal{C}$. Objects of this category assign a set of elements to each object of $\mathcal{C}$ together with restriction maps along morphisms of $\mathcal{C}$ \cite{MacLaneMoerdijk1994}. Presheaf topoi are particularly useful for representing structures that vary across contexts, since each object of $\mathcal{C}$ may be interpreted as a context and the associated presheaf assigns the corresponding set of admissible states or interpretations.

\subsection{A Worked Cognitive Example in a Presheaf Topos}
\label{sec:worked_presheaf_example}

In cognitive science and linguistics, context refers to the set of linguistic, situational, and background factors that influence the interpretation of an expression or event. Context provides the information necessary to disambiguate meanings, guide inference, and constrain possible interpretations \cite{SperberWilson1995,Clark1996}. More formally, context may be understood as the collection of assumptions, prior knowledge, and environmental cues that interact with a representation in order to determine its interpretation \cite{Levinson2000}. From a structural perspective, context acts as a system of constraints that restricts the space of admissible interpretations, progressively reducing ambiguity during the interpretative process.

To demonstrate that the proposed framework is not merely programmatic, we give a worked example inside a concrete topos.
A standard and widely used choice is a presheaf topos
\[
\widehat{\mathcal{C}} := \mathbf{Set}^{\mathcal{C}^{\mathrm{op}}},
\]
where $\mathcal{C}$ is a small category of contexts ordered by refinement. Presheaf topoi provide a canonical setting for contextual semantics and “varying sets” \cite{Johnstone2002,MacLaneMoerdijk1994}.


For the concrete application or our model we need to use  ordered sets that can be represented by the so-called poset-categories. A poset-category (or posetal category) is a category obtained from a partially ordered set (poset) by interpreting the order relation as morphisms between elements. More precisely, given a poset $(P,\leq)$, one may construct a category $\mathcal{C}$ whose objects are the elements of $P$ and for which there exists a unique morphism
$
x \rightarrow y
$
whenever $x \leq y$ \cite{MacLane1971, Awodey2010}. Composition of morphisms corresponds to the transitivity of the order relation, and identity morphisms correspond to reflexivity. Poset-categories provide a simple but important class of categories and are often used to model hierarchical structures or systems of contexts in which objects are related by refinement or inclusion.

Let $\mathcal{C}$ be the poset-category generated by three contexts
\[
C_0 \le C_1 \le C_2,
\]
with morphisms $C_2 \to C_1 \to C_0$ representing forgetting/pruning of contextual information (coarsening).
Interpretation:
\begin{itemize}
	\item $C_0$: minimal context (isolated word);
	\item $C_1$: sentence-level context;
	\item $C_2$: discourse/pragmatic context.
\end{itemize}


\subsubsection{Example}

Define a presheaf of admissible interpretations
\[
\mathcal{S} : \mathcal{C}^{\mathrm{op}} \to \mathbf{Set}
\]
for the ambiguous lexical item \emph{bank}:
\begin{align*}
\mathcal{S}(C_0) &= \{\textsf{financial},\ \textsf{river}\},\\
\mathcal{S}(C_1) &= \{\textsf{financial}\},\\
\mathcal{S}(C_2) &= \{\textsf{financial}\}.
\end{align*}
The restriction maps (contravariant action) enforce coherence under refinement:
\[
\rho_{10}:\mathcal{S}(C_0)\to \mathcal{S}(C_1), \qquad 
\rho_{21}:\mathcal{S}(C_1)\to \mathcal{S}(C_2),
\]
where $\rho_{10}$ discards the \textsf{river} reading once sentence-level context is available (e.g., {\it I went to the bank to open an account}).


Within $\widehat{\mathcal{C}}$, propositions are represented by subobjects; truth becomes context-indexed via the subobject classifier $\Omega$ \cite{bell2008,Johnstone2002}.
Consider the proposition:
\[
P := \text{``\emph{bank} refers to a financial institution.''}
\]
Then the characteristic morphism
\[
\chi_P : \mathcal{S} \to \Omega
\]
assigns a (Heyting-valued) truth at each context. Informally:
\begin{itemize}
	\item at $C_0$: $P$ is not decidable (two admissible readings);
	\item at $C_1, C_2$: $P$ is validated (only \textsf{financial} remains).
\end{itemize}
This realizes, inside a topos, the idea that semantic truth is not globally Boolean but depends on the available contextual stage \cite{MacLaneMoerdijk1994,bell2008}.


The transformation can be visualized by the following diagrams 
\[
\begin{tikzcd}[column sep=large]
C_2 \arrow[r] \arrow[d] & C_1 \arrow[d] \\
C_1 \arrow[r] & C_0
\end{tikzcd}
\qquad
\begin{tikzcd}[column sep=large]
\mathcal{S}(C_2)=\{\textsf{financial}\} \arrow[r,"\rho_{21}"] \arrow[d,equal] 
& \mathcal{S}(C_1)=\{\textsf{financial}\} \arrow[d,"\rho_{10}"] \\
\{\textsf{financial}\} \arrow[r] & \{\textsf{financial},\textsf{river}\}
\end{tikzcd}
\]

This example shows how lexical ambiguity resolution can be modeled as a \emph{context-refinement} phenomenon internal to a presheaf topos, rather than as an ad hoc external rule.

\begin{remark}
	Let us remind that the  \emph{Heyting algebra} is a bounded lattice equipped with an implication operation that models the logic of intuitionistic reasoning. 
		Formally, a Heyting algebra is a partially ordered set $(H,\leq)$ with finite meets and joins, together with a greatest element $1$, a least element $0$, and a binary operation	
		\[
		\rightarrow : H \times H \rightarrow H
		\]	
		such that for all $a,b,c \in H$,		
		\[
		a \wedge b \leq c 
		\quad \text{if and only if} \quad 
		a \leq (b \rightarrow c).
		\]	
		This condition characterizes the implication operator as the right adjoint of the meet operation \cite{Johnstone2002,MacLaneMoerdijk1994}. 
		Heyting algebras provide the algebraic semantics of intuitionistic logic and arise naturally in topos theory, where the lattice of subobjects of an object forms a Heyting algebra \cite{bell2008}.
\end{remark}


\subsection{A Topos-Theoretic Model of Contextual Cognition and coherence}
\label{sec:topos_contextual_cognition}

Category theory represents cognition as compositional transformation, but it does not by itself fix the logic governing truth and validity.
A topos provides a universe of sets with internal logic, typically intuitionistic, allowing truth to be modeled as context-sensitive rather than globally Boolean \cite{bell2008, Johnstone2002,MacLaneMoerdijk1994}.
This is structurally aligned with cognitive phenomena such as ambiguity, partial information, and defeasible interpretation.

A presheaf $P:\mathcal{C}^{op}\to\mathbf{Set}$ assigns to each context $U$ a set of admissible representations $P(U)$.
A \emph{sheaf} $\mathcal{F}$ enforces coherence under covers: compatible local representations glue uniquely to a global one \cite{Johnstone2002,MacLaneMoerdijk1994}.
In cognitive terms:
\begin{itemize}
	\item {\it Local consistency}: interpretations agree on overlaps;
	\item  {\it Global coherence}: compatible local interpretations determine a unified meaning.
\end{itemize}

 Sheafification behaves as concept/meaning stabilization.
When a presheaf contains fragmented or incompatible assignments, the sheafification functor
\[
a:\widehat{\mathcal{C}}\to \mathbf{Sh}(\mathcal{C},J)
\]
produces the closest coherent sheaf \cite{MacLaneMoerdijk1994}.
We interpret this as a mathematical idealization of cognitive stabilization:
	{ \emph Concepts (or meanings) arise as coherent gluing of locally constrained interpretations.}

In any topos, predicates correspond to subobjects and are classified by morphisms into the subobject classifier $\Omega$ \cite{Johnstone2002}.
Thus, “truth values” form a Heyting algebra rather than a Boolean algebra.
Cognitively, this models graded or context-indexed validity (e.g., “supported,” “undetermined,” “ruled out”) without forcing premature bivalence \cite{bell2008} as shown in the following diagram

\[
\begin{tikzcd}[column sep=large,row sep=large]
\widehat{\mathcal{C}}=\mathbf{Set}^{\mathcal{C}^{op}} \arrow[r,"a\ \text{(sheafification)}"] \arrow[d,swap,"\text{raw/fragmented}"] 
& \mathbf{Sh}(\mathcal{C},J) \arrow[d,"\text{coherent meaning}"] \\
P \arrow[r, dashed] & a(P)
\end{tikzcd}
\]

This expresses, in universal form, the transition from unconstrained representational variability to coherence under contextual constraints.


\section{Temporal and Dynamical Topoi for Cognitive Updating}
\label{sec:temporal_dynamical_topoi}

The fundamental property of our structural model of cognition  is the temporal  updating. In fact, 
cognition is a dynamical process,  intrinsically temporal: meanings drift, beliefs are revised, and interpretations stabilize or decay.
To avoid treating time as an external index, one may internalize temporal structure within a topos \cite{Johnstone2002, MacLaneMoerdijk1994}.
A canonical temporal topos is the \emph{topos of trees}
\[
\mathcal{T} := \mathbf{Set}^{\mathbb{N}^{op}},
\]
used widely in semantics of guarded recursion and step-indexed models \cite{BirkedalEtAl2011}.
 It plays a central role in modeling temporal, staged, and hierarchical structures, and has applications in logic, computer science, and semantics \cite{MacLaneMoerdijk1994,Awodey2010}.

It is defined as follows. 
Let $\mathbb{N}$ be the category whose objects are natural numbers
\[
0,1,2,\dots
\]
and where there exists a unique morphism
\[
n \to m \quad \text{if and only if} \quad n \le m.
\]

Then $\mathbb{N}$ is a poset category. Its opposite category $\mathbb{N}^{op}$ reverses the arrows, so that there is a morphism
\[
n \to m \quad \text{in } \mathbb{N}^{op} \quad \text{iff} \quad m \le n.
\]

The topos of trees is defined as the category of presheaves
\[
\mathbf{Set}^{\mathbb{N}^{op}} = \mathrm{Fun}(\mathbb{N}^{op}, \mathbf{Set}).
\]

An object $X$ in this category assigns:
\begin{itemize}
	\item to each $n \in \mathbb{N}$ a set $X_n$,
	\item to each morphism $n \to m$ (i.e.\ $m \le n$) a restriction map
	\[
	r_{n,m} : X_n \to X_m.
	\]
\end{itemize}

These maps satisfy the compatibility conditions:
\[
r_{n,n} = \mathrm{id}, 
\qquad
r_{m,k} \circ r_{n,m} = r_{n,k}.
\]

The structure of an object in $\mathbf{Set}^{\mathbb{N}^{op}}$ can be visualized as a tree.
 Elements of $X_n$ represent states or information available at level $n$.
 The restriction maps $r_{n,m}$ project information from finer levels to coarser ones. Thus, an element at a higher level determines a consistent family of approximations at all lower levels:
\[
x_n \mapsto (x_{n-1}, x_{n-2}, \dots, x_0).
\]

This hierarchical unfolding is why the category is called the \emph{topos of trees}.

The topos $\mathbf{Set}^{\mathbb{N}^{op}}$ naturally models temporal or staged processes: $n$ represents a stage, time step, or level of refinement,
  $X_n$ represents the cognitive or semantic state at stage $n$,
  restriction maps describe how future states determine past approximations.

In particular, a sequence
\[
x_0 \leftarrow x_1 \leftarrow x_2 \leftarrow \cdots
\]
represents a progressive refinement of meaning or cognition.

This makes the topos of trees particularly suitable for modeling:
 evolving semantic interpretations,
  context refinement,
  learning processes,
  recursive cognitive updates.

Like every topos, $\mathbf{Set}^{\mathbb{N}^{op}}$ has an internal intuitionistic logic \cite{MacLaneMoerdijk1994}. 
Truth values vary across levels, forming a Heyting algebra of stage-dependent truth.

Thus, a proposition is not simply true or false, but may evolve across stages:
\[
\varphi_0 \le \varphi_1 \le \varphi_2 \le \cdots
\]

This reflects the fact that knowledge and meaning may become more precise over time.

The cognitive update \eq{xi1}
can be interpreted internally in $\mathbf{Set}^{\mathbb{N}^{op}}$ as a transformation acting across stages.

Each stage $n$ corresponds to a level of refinement, and the restriction maps ensure coherence between stages. 
Thus, cognition may be viewed as a trajectory inside the topos of trees, where semantic meaning unfolds progressively and stabilizes across levels.

The topos of trees provides a natural mathematical setting in which:
 cognition is inherently temporal,
  meaning is context-dependent,
  knowledge evolves through refinement,
  and semantic stability emerges as coherence across levels.

This makes $\mathbf{Set}^{\mathbb{N}^{op}}$ a particularly suitable framework for modeling dynamical and hierarchical aspects of cognition.

An object $X\in \mathcal{T}$ consists of sets $X(n)$ (states at time $n$) and restriction maps
\[
r_{n+1,n}:X(n+1)\to X(n)
\]
encoding temporal coherence: later states restrict to consistent earlier approximations.
This naturally models stability of memory traces and diachronic constraint.

A learning/update process is an internal endomorphism
\[
L:X\to X
\]
in $\mathcal{T}$, with components $L_n:X(n)\to X(n)$ coherent with restrictions.
Stabilization corresponds to fixed points (when they exist):
\[
L(X)\cong X.
\]
This matches the conceptual role  assigned to recursion and stabilization in cognition.

A general state-based dynamic system can be expressed coalgebraically as
\[
X \to F(X),
\]
where $F$ encodes observations and next-state structure.
Coalgebra provides a mathematically clean bridge between dynamical systems and state-based computation.
Placed inside a temporal topos, it supports evolution with partiality and staged validity.

\[
\begin{tikzcd}[column sep=large,row sep=large]
X(n+1) \arrow[r,"L_{n+1}"] \arrow[d,"r_{n+1,n}"'] 
& X(n+1) \arrow[d,"r_{n+1,n}"] \\
X(n) \arrow[r,"L_n"'] & X(n)
\end{tikzcd}
\]

Commutativity expresses that learning/update respects temporal coherence.

Temporal topoi are particularly efficient since they unify 
  evolving representations ($X(n)$),
  coherence across time (restriction maps),
  learning (endomorphism $L$),
  state-based dynamics (coalgebra $X\to F(X)$),
within one categorical semantics \cite{BirkedalEtAl2011}.

\subsection{Illustrative dynamical graph of cognitive updating}

To provide a concrete visualization of the proposed update mechanism, consider the scalar  model
\begin{equation}
x_{t+1}=\Phi(x_t):=\beta \tanh\!\bigl(x_t+\alpha \sin x_t\bigr),
\label{eq:toy_dynamics}
\end{equation}
where $\alpha>0$ controls nonlinear contextual modulation and $0<\beta<1$ is a damping parameter.
The term $\sin x_t$ models nonlinear context-dependent perturbation, while the saturating nonlinearity
$\tanh$ reflects bounded semantic activation.

Figure~\ref{fig:dynamical_graph} shows the graph of the map $y=\Phi(x)$ together with the diagonal
$y=x$. Intersections correspond to fixed points of the cognitive update rule, that is, stabilized
interpretations.

\begin{figure}[htbp]
	\centering
	\begin{tikzpicture}
	\begin{axis}[
	width=11cm,
	height=7cm,
	domain=-3:3,
	samples=300,
	xlabel={$x_t$},
	ylabel={$x_{t+1}$},
	legend style={at={(0.98,0.02)},anchor=south east},
	grid=both,
	xmin=-3, xmax=3,
	ymin=-1.2, ymax=1.2
	]
	\addplot[thick, blue] {0.6*tanh(deg(x + 0.8*sin(deg(x))))};
	\addlegendentry{$x_{t+1}=\Phi(x_t)$}
	
	\addplot[thick, dashed] {x};
	\addlegendentry{$x_{t+1}=x_t$}
	\end{axis}
	\end{tikzpicture}
	\caption{Dynamical graph of the  cognitive update map $\Phi(x)=0.6\tanh(x+0.8\sin x)$. Fixed points are given by intersections with the diagonal.}
	\label{fig:dynamical_graph}
\end{figure}
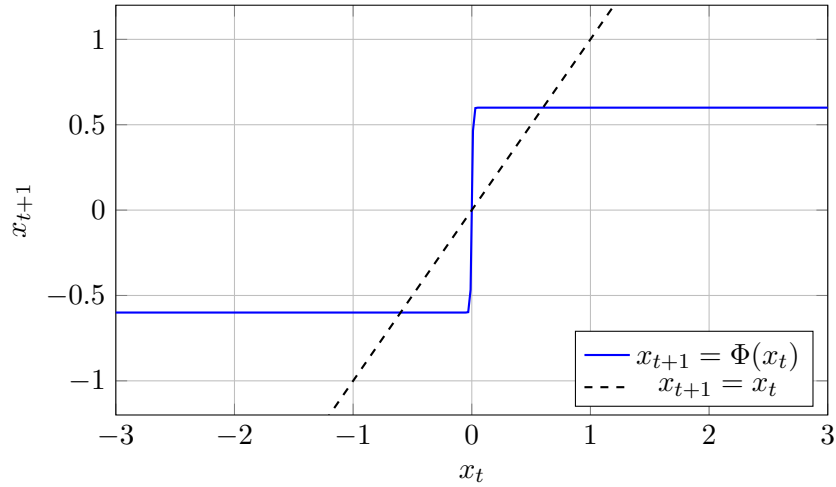

From a cognitive viewpoint, the graph illustrates three structural features of the model:
(i) boundedness of semantic activation, 
(ii) nonlinear contextual modulation, and
(iii) convergence toward stable interpretative states when iterates approach an attracting fixed point.


\paragraph{Connection with the Master Equation of Cognition.}

The presheaf example, the topos-theoretic framework, and the temporal dynamics developed in the previous sections can be unified through a single structural representation of cognitive updating. In the presheaf model, contextual refinement reduces interpretative ambiguity by restricting admissible representations. The topos-theoretic formulation generalizes this mechanism by interpreting cognitive representations as objects in a category endowed with an internal logic, allowing truth values to vary across contexts \cite{MacLaneMoerdijk1994,bell2008}. Temporal topoi further extend this perspective by modeling the evolution of representations across time through coherent restriction maps and internal update morphisms \cite{BirkedalEtAl2011}. These three perspectives correspond respectively to contextual refinement, semantic stabilization, and temporal evolution of cognitive states. Together they motivate the structural update rule of the master Eq.\eq{xi1}
where $X_t$ denotes the cognitive state at time $t$, $f$ represents cognitive transformations acting on representations, $F$ maps these transformations into semantic structures through functorial interpretation, and the equivalence relation $\sim$ identifies invariant meanings across representational variations. The master equation therefore summarizes the categorical architecture of cognition: representations evolve through compositional transformations, semantic invariants emerge through equivalence relations, and contextual coherence is enforced by the internal logic of the surrounding topos.




\section{Mathematical Framework of the Structural Theory of Cognition}

The structural theory proposed in this article models cognition as a system of transformations acting on representations under contextual and temporal constraints. The framework may be summarized through the following five axioms.

\begin{description}
	\item[Axiom 1: Logical Inference]   
	  Cognitive reasoning includes a logical component represented by propositional implication
\[
P \Rightarrow Q ,
\]
where $P$ denotes a hypothesis and $Q$ denotes a conclusion. Logical inference provides the symbolic structure through which information can be derived from previously established statements. This component corresponds to classical logical reasoning embedded within linguistic structures.

\item[Axiom 2: Transformational Representation]
Cognitive processes operate through transformations between representational states. Let $A$ and $B$ denote sets of cognitive states. A cognitive transformation is represented by a function
\[
f : A \rightarrow B .
\]
Thus cognition is modeled not as a static collection of propositions but as a system of transformations mapping one representational state into another.

\item[Axiom 3: Semantic Invariance]
Distinct representations may correspond to the same meaning. This property is modeled by an equivalence relation
\[
x \sim y ,
\]
satisfying reflexivity, symmetry, and transitivity. The equivalence relation groups representations into semantic classes that remain invariant under contextual variation.

\item[Axiom 4: Categorical Structure]
Cognitive representations and transformations form a category $\mathcal{C}$ whose objects are cognitive states and whose morphisms are cognitive transformations. Semantic interpretation is represented by a functor
\[
F : \mathcal{C} \rightarrow \mathcal{D},
\]
mapping cognitive representations to semantic structures while preserving compositional relationships between transformations.

\item[Axiom 5: Temporal Dynamics]
Cognition evolves over time through state transitions. Let $X_t$ denote the cognitive state at time $t$. The structural update rule governing cognitive evolution is given by

\[
X_{t+1} = F(f(X_t)) / \sim .
\]

Here $f$ represents cognitive transformation, $F$ represents functorial semantic mapping, and $\sim$ identifies invariant meanings across representations. This update rule describes cognition as a dynamic process in which logical inference, contextual interpretation, and semantic abstraction interact to produce evolving states of understanding.
\end{description}

Together, these axioms describe cognition as a compositional system of transformations evolving through contextual interaction while preserving semantic invariants. Logical reasoning provides symbolic structure, functions model cognitive transformations, equivalence relations capture semantic invariance, category theory organizes compositional relationships, and temporal dynamics represent the evolution of cognitive states.

This framework  may be summarized by the following categorical architecture:

\[
\begin{tikzcd}[column sep=huge]
X_t
\arrow[r, "f"]
& f(X_t)
\arrow[r, "F"]
& F(f(X_t))
\arrow[r, "/\sim"]
& X_{t+1}
\end{tikzcd}
\]

\[
\begin{tikzcd}[column sep=huge]
\mathcal{C}
\arrow[r, "F"]
& \mathcal{D}
\end{tikzcd}
\qquad
\begin{tikzcd}[column sep=huge]
X
\arrow[r]
& F(X)
\end{tikzcd}
\]

Here, \(X_t\) denotes the cognitive state at time \(t\), \(f\) represents cognitive transformation, \(F\) denotes functorial semantic mapping, and \(\sim\) identifies semantically equivalent representations. The category \(\mathcal{C}\) represents cognitive states and transformations, while \(\mathcal{D}\) represents semantic structures. The coalgebraic transition \(X \to F(X)\) expresses the temporal dynamics of cognitive updating.

In this representation, logical reasoning generates initial symbolic structures that are subsequently transformed through cognitive operations. Functorial mappings translate these structures into semantic domains while preserving their compositional relations. Equivalence relations identify invariant meanings across different representations. The surrounding topos provides the contextual environment in which interpretations are evaluated, while temporal morphisms represent the continuous updating of cognitive states. Together, these components describe cognition as a dynamic system of transformations evolving through contextual interaction while preserving structural semantic invariants.


\section{Cognitive dynamics and  stabilization }

 The axioms above imply that cognitive evolution operates on equivalence classes of representations rather than on individual representations. This observation can be formalized by the following theorem that  shows that the master equation of cognition defines a well-defined dynamical system on semantic equivalence classes.

\begin{theorem}[Well-defined cognitive dynamics]
	Let $X$ be the set of cognitive representations and let $\sim$ be the semantic equivalence relation defined in Axiom 3. Suppose that
	
	\[
	f : X \rightarrow X
	\]
	
	is a cognitive transformation (Axiom 2) and
	
	\[
	F : \mathcal{C} \rightarrow \mathcal{D}
	\]
	
	is the semantic functor defined in Axiom 4. If the transformation $F \circ f$ preserves semantic equivalence, that is,
	
	\[
	x \sim y \; \Rightarrow \; F(f(x)) \sim F(f(y)),
	\]
	
	then the update rule
	
	\[
	X_{t+1} = F(f(X_t)) / \sim
	\]
	
	defines a well-defined dynamical system on the quotient space
	
	\[
	X / \sim .
	\]
	
\end{theorem}

\begin{proof}
	Let $[x]$ denote the equivalence class of $x$ under the relation $\sim$.  
	By assumption, if $x \sim y$ then $F(f(x)) \sim F(f(y))$. Therefore the mapping
	
	\[	 	
	\Phi([x]) := [F(f(x))]
	\]
	does not depend on the choice of representative $x$ in the equivalence class $[x]$. 
 This is well-defined map provided that $F \circ f$ is compatible with the equivalence relation $\sim$.
 	Hence $\Phi$ defines a well-defined map
	
	\[
	\Phi : X/\sim \; \rightarrow \; X/\sim .
	\]
		This map describes the evolution of cognitive states at the level of semantic equivalence classes. Consequently, the cognitive update rule
	
	\[
	X_{t+1} = F(f(X_t)) / \sim
	\]
		induces a dynamical system on the quotient space $X/\sim$. This shows that cognitive evolution can be represented as the iteration of a transformation acting on semantic classes of representations rather than on individual representations. 
\end{proof}

This theorem shows that the  master equation does not depend on individual representations, but on semantic classes of representations so that cognition evolves on meanings,
not on raw representations

With the following fixed point theorem for cognitive stabilization  we show the  conditions under which cognition converges to a stable interpretation, thus avoiding 
linguistic disambiguation, and leading to belief stabilization
and semantic convergence.
In the categorical setting, semantic stabilization may be interpreted as the existence of a fixed point for the endomorphism induced by the cognitive update process on the quotient object of semantic equivalence classes.
The structural update rule introduced above suggests that cognitive evolution may converge toward stable semantic states. This intuition can be formalized as a fixed point result on the quotient space of semantic equivalence classes.

\begin{theorem}[Fixed Point Theorem for Cognitive Stabilization]
	Let $X$ be a nonempty set of cognitive representations, let $\sim$ be a semantic equivalence relation on $X$, and let
	\[
	\Phi : X/\sim \;\to\; X/\sim
	\]
	be the induced cognitive update map defined by
	\[
	\Phi([x]) = [F(f(x))],
	\]
	where $f : X \to X$ is a cognitive transformation and $F$ is the semantic functor. Assume that:
	
	\begin{enumerate}
		\item the map $\Phi$ is well defined;
		\item there exists a metric $d$ on $X/\sim$ such that $(X/\sim,d)$ is a complete metric space;
		\item $\Phi$ is a contraction, i.e. there exists a constant $0<c<1$ such that
		\[
		d(\Phi([x]),\Phi([y])) \le c\, d([x],[y])
		\qquad
		\text{for all } [x],[y]\in X/\sim.
		\]
	\end{enumerate}
	
	Then $\Phi$ admits a unique fixed point $[x^\ast]\in X/\sim$, that is,
	\[
	\Phi([x^\ast]) = [x^\ast].
	\]
	Moreover, for every initial cognitive state $[x_0]\in X/\sim$, the iterated sequence
	\[
	[x_{n+1}] = \Phi([x_n])
	\]
	converges to $[x^\ast]$.
	
\end{theorem}

\begin{proof}
	By assumption, $(X/\sim,d)$ is a complete metric space and $\Phi$ is a contraction on $X/\sim$. Therefore the Banach Fixed Point Theorem applies. Hence there exists a unique element $[x^\ast]\in X/\sim$ such that
	\[
	\Phi([x^\ast]) = [x^\ast].
	\]
	
	It remains to interpret the iteration. Starting from any initial class $[x_0]\in X/\sim$, define recursively
	\[
	[x_{n+1}] = \Phi([x_n]).
	\]
	Again by the Banach Fixed Point Theorem, the sequence $\{[x_n]\}_{n\ge 0}$ converges in the metric $d$ to the unique fixed point $[x^\ast]$.
	
	Thus the cognitive update dynamics stabilize to a unique semantic equivalence class, independently of the initial representative, provided that the contraction hypothesis holds.
\end{proof}

\begin{remark}
	The theorem shows that, under suitable regularity assumptions, repeated contextual updating leads to a stable interpretation. In cognitive terms, this fixed point represents a semantically stabilized state: ambiguity is progressively reduced, and the interpretive process converges toward an invariant meaning class.
\end{remark}

\subsection{Numerical stability of a the case cognitive update}

To complement the abstract fixed-point result, we analyze the numerical stability of the case map
\begin{equation}
\Phi(x)=\beta \tanh\!\bigl(x+\alpha \sin x\bigr),
\label{eq:toy_phi}
\end{equation}
with parameters $\alpha>0$ and $\beta>0$.

\begin{proposition}[Sufficient condition for global stability]
	Assume that
	\begin{equation}
	\beta(1+\alpha)<1.
	\label{eq:contraction_condition}
	\end{equation}
	Then $\Phi:\mathbb{R}\to\mathbb{R}$ is a contraction. Consequently, the iteration
	\[
	x_{t+1}=\Phi(x_t)
	\]
	admits a unique fixed point $x^\ast\in\mathbb{R}$, and for every initial condition $x_0\in\mathbb{R}$
	the sequence $(x_t)_{t\geq 0}$ converges to $x^\ast$.
\end{proposition}

\begin{proof}
	Differentiating \eqref{eq:toy_phi}, we obtain
	\[
	\Phi'(x)=\beta\,\sech^2\!\bigl(x+\alpha\sin x\bigr)\bigl(1+\alpha\cos x\bigr).
	\]
	Hence
	\[
	|\Phi'(x)|
	\leq
	\beta \cdot 1 \cdot (1+\alpha)
	=
	\beta(1+\alpha)
	\qquad \text{for all }x\in\mathbb{R},
	\]
	because $\sech^2(z)\leq 1$ and $|1+\alpha\cos x|\leq 1+\alpha$.
	If $\beta(1+\alpha)<1$, then
	\[
	|\Phi'(x)|\leq c<1
	\qquad \text{for all }x\in\mathbb{R},
	\]
	so $\Phi$ is globally Lipschitz with constant $c<1$, hence a contraction.
	The conclusion follows from the Banach fixed-point theorem.
\end{proof}

A particularly simple admissible choice is
\[
\alpha=0.8,
\qquad
\beta=0.5,
\]
for which
\[
\beta(1+\alpha)=0.5(1.8)=0.9<1.
\]
Therefore the corresponding cognitive dynamics are globally stable.

From the modeling viewpoint, condition \eqref{eq:contraction_condition} expresses a balance between
two competing effects: $\alpha$ amplifies contextual nonlinearity, whereas $\beta$ damps the update.
Stability is guaranteed when damping dominates contextual amplification.

\[
\begin{tikzcd}[column sep=large,row sep=large]
X \arrow[r,"f"] \arrow[d,"\pi"'] 
& X \arrow[r,"F"] 
& X \arrow[d,"\pi"] \\
X/\sim \arrow[rr,"\Phi"'] 
& 
& X/\sim
\end{tikzcd}
\]

The commutativity of this diagram expresses that the cognitive update rule is compatible with semantic equivalence. Consequently, the cognitive dynamics are well defined on the quotient space of semantic classes.

\section{Adjunction between Syntax and Semantics}

	Syntax refers to the system of rules and principles that govern the structure and organization of expressions in a language. It specifies how words and symbols combine to form well-formed phrases and sentences, independently of their meaning \cite{Chomsky1957,Carnie2013}. In formal approaches, syntactic structures are typically represented through hierarchical constructions such as trees or derivations that describe the compositional arrangement of linguistic elements \cite{Chomsky1995}. From a structural perspective, syntax therefore defines the formal architecture of linguistic representations upon which semantic interpretation operates.

The interaction between syntactic representations and semantic interpretation may be described categorically through an adjunction between categories.

An \emph{adjunction} between two categories $\mathcal{C}$ and $\mathcal{D}$ consists of a pair of functors

\[
F : \mathcal{C} \rightarrow \mathcal{D}, 
\qquad
G : \mathcal{D} \rightarrow \mathcal{C}
\]
together with a natural bijection between hom-sets

\[
\mathrm{Hom}_{\mathcal{D}}(F(X),Y)
\;\cong\;
\mathrm{Hom}_{\mathcal{C}}(X,G(Y))
\]
that is natural in both $X \in \mathcal{C}$ and $Y \in \mathcal{D}$ \cite{Awodey2010,MacLane1971}. 
In this situation the functor $F$ is called the \emph{left adjoint} of $G$, and $G$ is called the \emph{right adjoint} of $F$, often denoted

\[
F \dashv G.
\]

Adjunctions express a fundamental correspondence between two mathematical structures, allowing constructions in one category to be systematically translated into constructions in another. They play a central role in many areas of mathematics, including algebra, topology, logic, and category theory \cite{MacLane1971}.

\begin{remark} 
Let us remind that for a category $\mathcal{C}$ and objects $A,B \in \mathcal{C}$, the notation
$
\mathrm{Hom}_{\mathcal{C}}(A,B)
$
denotes the set of morphisms from $A$ to $B$ in the category $\mathcal{C}$ \cite{Awodey2010,MacLane1971}. 
Elements of this set are arrows
$
f : A \rightarrow B
$
representing structure-preserving transformations between the objects.
Hom-sets play a central role  because many constructions can be expressed in terms of relationships between these sets.  In particular, adjunctions between categories are defined through natural correspondences between Hom-sets.
Let
$
F : \mathcal{C} \rightarrow \mathcal{D},
\qquad
G : \mathcal{D} \rightarrow \mathcal{C}
$
be two functors. An adjunction between $\mathcal{C}$ and $\mathcal{D}$ consists of a natural bijection
$
\mathrm{Hom}_{\mathcal{D}}(F(X),Y)
\;\cong\;
\mathrm{Hom}_{\mathcal{C}}(X,G(Y))
$
for all objects $X \in \mathcal{C}$ and $Y \in \mathcal{D}$ \cite{MacLane1971}. 
When this condition holds, $F$ is called the left adjoint of $G$ and one writes
$
F \dashv G.
$
Intuitively, this correspondence expresses that morphisms from $F(X)$ to $Y$ in the category $\mathcal{D}$ correspond uniquely to morphisms from $X$ to $G(Y)$ in the category $\mathcal{C}$.
\end{remark}

\begin{remark} 
	The symbol $\cong$ between morphisms generally denotes that the morphisms correspond under an isomorphism or a natural bijection between Hom-sets. 
	More precisely, suppose that
	$
	\Phi : \mathrm{Hom}_{\mathcal{D}}(A,B)
	\rightarrow
	\mathrm{Hom}_{\mathcal{C}}(C,D)
	$	
	is a bijection between two sets of morphisms. If
$
	f \in \mathrm{Hom}_{\mathcal{D}}(A,B),
	\qquad
	g \in \mathrm{Hom}_{\mathcal{C}}(C,D),
$
	and $\Phi(f)=g$, we writes
$
	f \cong g
$
		to indicate that the two morphisms correspond through this bijection \cite{Awodey2010,MacLane1971}. 
	Thus the symbol $\cong$ expresses that the morphisms are not necessarily identical but are equivalent through a structural correspondence induced by an isomorphism of objects, a natural transformation, or an adjunction.	
\end{remark}

Now we can show that the adjunction between the category of syntactic representations $\mathcal{S}$ and the category of semantic structures $\mathcal{M}$ is given as follows: 

\[
\begin{tikzcd}[column sep=large]
\mathcal{S} 
\arrow[r, shift left=2, "F"] 
& 
\mathcal{M}
\arrow[l, shift left=2, "G"]
\end{tikzcd}
\]

\[
F \dashv G
\]


  The functor $F$ interprets syntactic constructions as semantic objects, while $G$ associates semantic structures with corresponding syntactic representations. The adjunction $F \dashv G$ expresses the structural correspondence between language and meaning, i.e.  between syntax and semantics. The functor $F$ interprets syntax into meaning while the functor $G$ reconstructs the syntactic structure from the semantic constraints. So that adjunction expresses the balance between structure and interpretation in cognition.  In this framework, $\mathcal{S}$ represents the category of syntactic expressions and transformations, while $\mathcal{M}$ represents the category of semantic structures. The functor $F$ maps syntactic expressions into their semantic interpretations, whereas $G$ associates semantic objects with their possible syntactic realizations. The adjunction $F \dashv G$  can be expressed by  the natural bijection  
\[
\mathrm{Hom}_{\mathcal{M}}(F(S),M)
\cong
\mathrm{Hom}_{\mathcal{S}}(S,G(M)).
\]

By using the adjunction between syntax and semantics, it can be shown that

\begin{theorem}[Cognitive Adjunction Principle]
	Let $\mathcal{S}$ denote a category of syntactic or cognitive representations and let $\mathcal{M}$ denote a category of semantic meanings. 
	Suppose there exist functors
	
	\[
	F : \mathcal{S} \rightarrow \mathcal{M},
	\qquad
	G : \mathcal{M} \rightarrow \mathcal{S}
	\]
	
	such that
	
	\[
	\mathrm{Hom}_{\mathcal{M}}(F(X),Y)
	\;\cong\;
	\mathrm{Hom}_{\mathcal{S}}(X,G(Y))
	\]
	
	naturally in $X \in \mathcal{S}$ and $Y \in \mathcal{M}$. 
	Then interpretation and representation form an adjoint pair
	
	\[
	F \dashv G,
	\]
	
	expressing a structural correspondence between syntactic representations and semantic meanings.
\end{theorem}

\begin{proof}
	A morphism 
$
	F(X) \rightarrow Y
$
	in the semantic category represents a way in which the meaning of the representation $X$ can be interpreted as the semantic object $Y$. 
	
	Conversely, a morphism
	$
	X \rightarrow G(Y)
	$
	represents a way in which the semantic structure $Y$ can be realized or encoded syntactically.
		The natural bijection between the two Hom-sets therefore establishes a one-to-one correspondence between semantic interpretations and syntactic realizations. 
	This correspondence satisfies the defining property of an adjunction, yielding the adjoint pair
	$
	F \dashv G.
$
\end{proof}

  Intuitively this 
 theorem states: syntax $\rightarrow $   semantics  via $F$ and 
 semantics $\rightarrow $  syntax  via $G$
and adjunction means: interpreting a structure and generating a structure from meaning are dual operations, according to the structural bridge between language and cognition given by our  framework.


The interaction between syntactic representations and semantic interpretation may be characterized through a universal property. Intuitively, semantic meaning provides a canonical representation that factors all meaning-preserving interpretations of syntactic expressions.

	Let $\mathcal{S}$ be the category of syntactic representations and $\mathcal{M}$ the category of semantic structures. Suppose that
	$
	F : \mathcal{S} \rightarrow \mathcal{M}
$	
	is the semantic interpretation functor, we
can show that 
	
	\begin{theorem}[Universal Property of Meaning] For every category $\mathcal{D}$ and every functor
	
	\[
	H : \mathcal{S} \rightarrow \mathcal{D}
	\]
	
	that preserves semantic equivalence, there exists a unique functor
	
	\[
	\widetilde{H} : \mathcal{M} \rightarrow \mathcal{D}
	\]
	
	such that the following diagram commutes:
	
	\[
	\begin{tikzcd}[column sep=large,row sep=large]
	\mathcal{S} 
	\arrow[r,"F"] 
	\arrow[dr,"H"'] 
	& 
	\mathcal{M} 
	\arrow[d,"\widetilde{H}"] \\
	& 
	\mathcal{D}
	\end{tikzcd}
	\]
	
	That is,
	
	\[
	H = \widetilde{H} \circ F.
	\]
	
\end{theorem}

\begin{proof}
	The functor $F$ assigns to each syntactic representation its semantic interpretation. If a functor $H$ preserves semantic equivalence, then the value of $H$ on syntactic objects depends only on their semantic interpretation. Consequently, $H$ factors uniquely through the semantic functor $F$. This factorization defines the functor $\widetilde{H}$ such that $H = \widetilde{H} \circ F$, and the uniqueness follows from the universality of the semantic interpretation.
	\end{proof}

	This proposition expresses that semantic meaning acts as a universal mediator between syntactic representations and other interpretative structures. Any interpretation of syntactic expressions that respects semantic equivalence must factor through the semantic interpretation functor. In cognitive terms, this means that meaning provides the canonical level at which different representations can be compared and integrated.

So that  every interpretation of language that respects meaning
must pass through the semantic level. In short:
meaning is the universal mediator between representations.


\section{Extended Example: Cognitive Interpretation of an Ambiguous Text}

To illustrate the structural theory of cognition developed in the previous sections, we consider a concrete lexical example involving lexical ambiguity and contextual interpretation.

Consider the (ambiguous) lexical sentence

\[
S_0 = \text{``I went to the bank.''}
\]
At the purely syntactic level, the word ``bank'' admits multiple interpretations. Let
\[
X = \{\textsf{financial\ institution},\ \textsf{river\ bank}\}
\]
denote the set of possible semantic interpretations associated with this lexical item.

Thus the initial cognitive state may be represented as
\[
X_0 = X .
\]

This corresponds to the situation in which the listener has not yet received sufficient contextual information to determine the intended meaning.

Now consider the extended sentence
\[
S_1 = \text{``I went to the bank to open a savings account.''}
\]
The contextual phrase ``to open a savings account'' provides additional semantic constraints that restrict the admissible interpretations. Formally, we define the cognitive transformation
\[
f : X \rightarrow X
\]
that incorporates contextual information and refinement. Under this transformation,

\[
f(\textsf{financial\ institution}) = \textsf{financial\ institution},
\]

\[
f(\textsf{river\ bank}) = \emptyset.
\]

Thus the updated cognitive state becomes

\[
X_1 = f(X_0) = \{\textsf{financial\ institution}\}.
\]

Let
\[
F : \mathcal{C} \rightarrow \mathcal{M}
\]
be the semantic functor mapping cognitive representations to semantic structures. Applying this functor yields

\[
F(f(X_0)) = \textsf{financial\ institution}.
\]

This step corresponds to interpreting the syntactic representation within the semantic domain, so that $F$ gives the semantic interpretation.
However, this interpretation, although unique, may arise from multiple equivalent representations.
  In fact, different linguistic expressions may correspond to the same semantic meaning.
For example,

\[
S_1 = \text{``I went to the bank to open a savings account.''}
\]
and
\[
S_2 = \text{``I visited a financial institution to open an account.''}
\]

represent the same semantic content. This relation is captured by the equivalence relation
\[
S_1 \sim S_2 .
\]
The equivalence class
\[
[S] = \{\text{all sentences expressing the same financial interpretation}\}
\]
represents the invariant semantic meaning associated with the utterance.

The entire interpretative process can therefore be summarized by the cognitive update master  equation \eq{xi1}. 

So that, starting from the ambiguous state
\[
X_0 = \{\textsf{financial},\ \textsf{river}\},
\]
the contextual transformation $f$ eliminates incompatible interpretations, the semantic functor $F$ maps the resulting representation into the semantic domain, and the equivalence relation $\sim$ identifies all syntactic expressions that share the same meaning.

All this process may be summarized by the following categorical diagram:

\[
\begin{tikzcd}[column sep=large]
X_0
\arrow[r,"f"]
&
f(X_0)
\arrow[r,"F"]
&
F(f(X_0))
\arrow[r,"/\sim"]
& \left[X\right]
\end{tikzcd}
\]
Here $[X]$ denotes the semantic equivalence class corresponding to the stabilized interpretation.

This example illustrates how the proposed framework models cognitive interpretation as a sequence of transformations acting on representational states. Initially, multiple semantic possibilities coexist. Contextual information reduces this ambiguity by restricting the admissible interpretations. The semantic mapping organizes these interpretations within a structured semantic domain, while equivalence relations identify invariant meanings across different linguistic realizations.

Thus the cognitive interpretation of a sentence may be described as a dynamic process in which contextual information, semantic mapping, and equivalence relations interact to produce a stable semantic interpretation.

 \subsection{Concrete linguistic application: context-sensitive interpretation of an ambiguous utterance}
 
 A more concrete linguistic application of the proposed framework is given by context-sensitive
 interpretation of the sentence
 \[
 S_0 = \text{``The chicken is ready to eat.''}
 \]
 This utterance is classically ambiguous, since it admits at least two interpretations:
 
 \[
 m_1 = \text{``the cooked chicken is ready to be eaten,''}
 \]
 \[
 m_2 = \text{``the live chicken is ready to eat food.''}
 \]
 
 Let the initial interpretative state be
 \[
 X_0=\{m_1,m_2\}.
 \]
 Without contextual constraints, both meanings remain admissible.
 
 Now consider two different contextual continuations.
 
 \paragraph{Case A: restaurant context.}
 Suppose the utterance is followed by
 \[
 C_A=\text{``Please bring the plates.''}
 \]
 This context favors the food interpretation. We model the corresponding cognitive transformation
 $f_A$ by
 \[
 f_A(m_1)=m_1,
 \qquad
 f_A(m_2)=\varnothing.
 \]
 Hence
 \[
 X_1^{(A)}=f_A(X_0)=\{m_1\}.
 \]
 
 \paragraph{Case B: farm context.}
 Suppose instead that the utterance is followed by
 \[
 C_B=\text{``It has been waiting for its grain since morning.''}
 \]
 This context favors the animal-agent interpretation. We model the corresponding cognitive
 transformation $f_B$ by
 \[
 f_B(m_1)=\varnothing,
 \qquad
 f_B(m_2)=m_2.
 \]
 Hence
 \[
 X_1^{(B)}=f_B(X_0)=\{m_2\}.
 \]
 
 Let $F$ denote the semantic interpretation map. Then
 \[
 F(X_1^{(A)})=m_1,
 \qquad
 F(X_1^{(B)})=m_2.
 \]
 The final semantic outcome depends on contextual updating, not on syntax alone.
 
 This example shows that the same surface sentence may generate different semantic fixed points
 depending on the context in which it is processed. In the language of the present framework,
 cognitive interpretation is a context-driven trajectory from an ambiguous representational state
 to a stabilized semantic class.
 
 The example is especially relevant for cybernetic and cognitive modeling because it exhibits
 feedback between incoming contextual information and semantic state revision. The utterance is
 not interpreted once and for all at the syntactic level; rather, its meaning is progressively regulated
 through dynamic interaction with contextual input.

\section*{Conclusion}

We have proposed a structural theory of cognition in which cognitive-linguistic processes are modeled as transformations between states of interpretation. Functions capture procedural transformations, equivalence relations encode semantic invariance, and category theory provides the compositional structure linking cognitive operations.

The framework yields a mathematically rigorous representation of cognition as a dynamical system evolving under contextual and temporal constraints. It also suggests new directions for integrating logic, semantics, and dynamical systems within a unified categorical setting.

Future work may address quantitative stability, stochastic extensions, and applications to artificial intelligence and neural representations.


\begin{thebibliography}{99}
	
	\bibitem{lawvere1963}
	F. W. Lawvere,
	\textit{Functorial Semantics of Algebraic Theories},
	Proceedings of the National Academy of Sciences, 50 (1963), 869--872.
	
	\bibitem{bell2008}
	J. L. Bell,
	\textit{Toposes and Local Set Theories},
	Dover, 2008.
	
	
	\bibitem{clark1998}
	A. Clark,
	\textit{Being There: Putting Brain, Body, and World Together Again},
	MIT Press, 1998.
	
	\bibitem{lawvere2009}
	F. W. Lawvere and S. Schanuel,
	\textit{Conceptual Mathematics: A First Introduction to Categories},
	Cambridge University Press, 2009.
	
	\bibitem{MacLaneMoerdijk1994}
	S. Mac Lane and I. Moerdijk,
	\textit{Sheaves in Geometry and Logic: A First Introduction to Topos Theory},
	Springer, 1994.
	
	\bibitem{Johnstone2002}
	P. T. Johnstone,
	\textit{Sketches of an Elephant: A Topos Theory Compendium},
	Oxford University Press, 2002.
	
	\bibitem{BirkedalEtAl2011}
	L. Birkedal et al.,
	\textit{First Steps in Synthetic Guarded Domain Theory},
	Logical Methods in Computer Science, 2011.
	
	
	\bibitem{Neisser1967}
	U. Neisser,
	\textit{Cognitive Psychology},
	Appleton-Century-Crofts, 1967.
	
	\bibitem{Anderson2015}
	J. R. Anderson,
	\textit{Cognitive Psychology and Its Implications},
	Worth Publishers, 2015.
	
	\bibitem{Friston2010}
	K. Friston,
	\textit{The Free-Energy Principle: A Unified Brain Theory?},
	Nature Reviews Neuroscience, 11 (2010), 127--138.
	
	\bibitem{Clark1996}
	H. H. Clark,
	\textit{Using Language},
	Cambridge University Press, 1996.
	
	\bibitem{Levinson2000}
	S. C. Levinson,
	\textit{Presumptive Meanings},
	MIT Press, 2000.
	
	\bibitem{SperberWilson1995}
	D. Sperber and D. Wilson,
	\textit{Relevance: Communication and Cognition},
	Blackwell, 1995.
	
	\bibitem{Lyons1977}
	J. Lyons,
	\textit{Semantics},
	Cambridge University Press, 1977.
	
	\bibitem{ChierchiaMcConnellGinet2000}
	G. Chierchia and S. McConnell-Ginet,
	\textit{Meaning and Grammar},
	MIT Press, 2000.
	
	\bibitem{DowtyWallPeters1981}
	D. Dowty, R. Wall, and S. Peters,
	\textit{Introduction to Montague Semantics},
	D. Reidel, 1981.
	
	\bibitem{Chomsky1957}
	N. Chomsky,
	\textit{Syntactic Structures},
	Mouton, 1957.
	
	\bibitem{Carnie2013}
	A. Carnie,
	\textit{Syntax: A Generative Introduction},
	Wiley-Blackwell, 2013.
	
	\bibitem{Chomsky1995}
	N. Chomsky,
	\textit{The Minimalist Program},
	MIT Press, 1995.
	
	\bibitem{MacLane1971}
	S. Mac Lane,
	\textit{Categories for the Working Mathematician},
	Springer, 1971.
	
	\bibitem{Awodey2010}
	S. Awodey,
	\textit{Category Theory},
	Oxford University Press, 2010.
	
	\bibitem{Rutten2000}
	J. J. M. M. Rutten,
	\textit{Universal Coalgebra: A Theory of Systems},
	Theoretical Computer Science, 249 (2000), 3--80.
	
	\bibitem{Halmos1963}
	P. R. Halmos,
	\textit{Lectures on Boolean Algebras},
	Van Nostrand, 1963.
	
	\bibitem{DaveyPriestley2002}
	B. A. Davey and H. A. Priestley,
	\textit{Introduction to Lattices and Order},
	Cambridge University Press, 2002.
	
	\bibitem{Halmos1960}
	P. R. Halmos,
	\textit{Naive Set Theory},
	Van Nostrand, 1960.
	
	\bibitem{moca21}
	D. Monte-Serrat and C. Cattani,
	\textit{The Natural Language for Artificial Intelligence},
	Academic Press, 2021.
	
	\bibitem{ecc13}
	P. J. Eccles,
	\textit{An Introduction to Mathematical Reasoning},
	Cambridge University Press, 2013.
	
	\bibitem{moca21a}
	D. Monte-Serrat and C. Cattani,
	\textit{Interpretability in Neural Networks Towards Universal Consistency},
	International Journal of Cognitive Computing in Engineering, 2 (2021), 30--39.
	
	\bibitem{der68}
	J. Derrida,
	\textit{De la Grammatologie},
	Les Éditions de Minuit, 1968.
	
	\bibitem{who12}
	B. L. Whorf,
	\textit{Language, Thought, and Reality},
	MIT Press, 2012.
	
	\bibitem{alo20}
	F. R. Al-Osaimi,
	\textit{Learning Descriptors Invariance Through Equivalence Relations},
	Journal of Imaging, 6 (2020), 120.
	
	\bibitem{per07}
	L. Perlovsky and R. Kozma (eds.),
	\textit{Neurodynamics of Cognition and Consciousness},
	Springer, 2007.
	
	\bibitem{Quigley2024}
	D. Quigley,
	\textit{Categorical Framework for Typed Extensional and Intensional Models in Formal Semantics},
	arXiv, 2024.
	
	\bibitem{Phillips2024}
	S. Phillips,
	\textit{A Category Theory Perspective on the Language of Thought},
	Frontiers in Psychology, 15 (2024).
	
	\bibitem{Prentner2024}
	R. Prentner,
	\textit{Category Theory in Consciousness Science},
	Synthese, 204 (2024).
	
	\bibitem{Lambert2024}
	M. J. Lambert,
	\textit{A Topos-Theoretic Semantics of Intuitionistic Modal Logic},
	arXiv, 2024.
	
	\bibitem{Lewis2024}
	M. Lewis,
	\textit{Grounded Learning for Compositional Vector Semantics},
	arXiv, 2024.
	
	\bibitem{Wojtowicz2025}
	R. Wojtowicz,
	\textit{Logic-Based Artificial Intelligence Algorithms Supporting Categorical Semantics},
	arXiv, 2025.
	
	\bibitem{Bakirtzis2025}
	G. Bakirtzis et al.,
	\textit{Categorical Semantics of Compositional Reinforcement Learning},
	Journal of Machine Learning Research, 26 (2025).
	
	\bibitem{Mahadevan2025}
	S. Mahadevan,
	\textit{Topos Theory for Generative AI and Large Language Models},
	arXiv, 2025.
	
	
\end{thebibliography}

\end{document}